\documentclass[10pt,twocolumn,letterpaper]{article}

\usepackage{cvpr}
\usepackage{times}
\usepackage{epsfig}
\usepackage{graphicx}
\usepackage{amsmath}
\usepackage{amssymb}
\usepackage[inline]{enumitem}
\usepackage{booktabs}
\usepackage{verbatim}
\usepackage{algorithm,caption}
\usepackage[noend]{algorithmic}
\usepackage[numbers]{natbib}
\usepackage{proof}
\usepackage{amsthm}
\newtheorem{theorem}{Theorem}
\usepackage[detect-weight=true]{siunitx}
\usepackage{bm}
\usepackage[table]{xcolor}
\usepackage{xcolor}
\usepackage{moresize}

\usepackage{microtype}
\usepackage[pagebackref=true,breaklinks=true,letterpaper=true,colorlinks,bookmarks=false]{hyperref}
\usepackage{url}

\cvprfinalcopy 

\def\cvprPaperID{3741} 

\ifcvprfinal\fi

\newif\ifdraft
\drafttrue

\ifdraft
  \newcommand{\ksenia}[1]{{\color{cyan}{#1}}}
  \newcommand{\christoph}[1]{{\color{blue}{C: #1}}}
  \newcommand{\jasper}[1]{{\color{magenta}{J: #1}}}  
  \newcommand{\vitto}[1]{{\color{red}{V: #1}}}
\else
  \newcommand{\ksenia}[1]{}
  \newcommand{\christoph}[1]{}
  \newcommand{\jasper}[1]{}
  \newcommand{\vitto}[1]{}
\fi

\newcommand{\IAD}{{IAD}}

\newcommand{\approachClassification}{{IAD-Prob}}
\newcommand{\approachRL}{{IAD-RL}}

\newcommand\descitem[1]{\item{\bfseries #1}\\}

\newcommand{\mypartop}[1]{\vspace{0mm}\paragraph{#1}}
\newcommand{\mypar}[1]{\vspace{-4mm}\paragraph{#1}}

\begin{document}

\title{Learning Intelligent Dialogs for Bounding Box Annotation}

\author{Ksenia Konyushkova\\
\small{CVLab, EPFL} \\
{\tt\scriptsize ksenia.konyushkova@epfl.ch}
\thanks{This work was done during an internship at Google AI Perception}
\and
Jasper Uijlings\\
\small{Google AI Perception}\\
{\tt\scriptsize jrru@google.com}
\and
Christoph H. Lampert\\
\small{IST Austria}\\
{\tt\scriptsize chl@ist.ac.at}
\and
Vittorio Ferrari\\
\small{Google AI Perception}\\
{\tt\scriptsize vittoferrari@google.com}
}

\maketitle
\graphicspath{{images-files/}}

\begin{abstract}

We introduce Intelligent Annotation Dialogs for bounding box annotation.
We train an agent to automatically choose a sequence of actions for a human annotator to produce a bounding box
in a minimal amount of time.  Specifically, we consider two actions:
box verification~\cite{papadopoulos16cvpr}, where the annotator verifies a box generated by an object
detector, and manual box drawing.
We explore two kinds of agents, one based on predicting the probability that a box will be positively verified,
and the other based on reinforcement learning.
We demonstrate that
(1) our agents are able to learn efficient annotation strategies in several scenarios,
automatically adapting to the image difficulty, the desired quality of the boxes, and the detector strength;
(2) in all scenarios the resulting annotation dialogs speed up annotation compared to manual box drawing alone and box verification alone, while also outperforming any fixed combination of verification and drawing in most scenarios;
(3) in a realistic scenario where the detector is iteratively re-trained, our agents evolve a series of strategies that reflect the shifting trade-off between verification and drawing as the detector grows stronger.

\end{abstract}

\section{Introduction}
\label{sec:introduction}

Many recent advances in computer vision rely on supervised machine learning techniques that are known to crave for huge amounts of training data.
Object detection is no exception as state-of-the-art methods require a large number of images with annotated bounding boxes around objects.
However, drawing high quality bounding boxes is expensive: The official protocol used to annotate ILSVRC~\cite{russakovsky15ijcv} takes about 30 seconds per box~\cite{su12aaai}.
To reduce this cost, recent works explore cheaper forms of human supervision such as image-level labels~\cite{bilen16cvpr,kantorov16eccv,zhu17iccv}, box verification series~\cite{papadopoulos16cvpr}, point annotations~\cite{mettes16eccv,papadopoulos17cvpr}, and eye-tracking~\cite{papadopoulos14eccv}.

Among these forms, the recent work on box verification series~\cite{papadopoulos16cvpr} stands out as it demonstrated to deliver high quality detectors at low cost.
The scheme starts from a given weak detector, typically trained on image labels only, and uses it to localize objects in the images. For each image, the annotator is asked to verify whether the box produced by the algorithm covers an object tightly enough.
If not, the process iterates: the algorithm proposes another box and the annotator verifies it. 

\begin{figure}[t]
  \vspace{-0.2cm}
\includegraphics[width=0.46\linewidth]{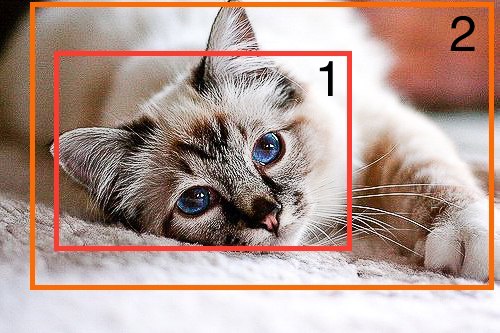}
\includegraphics[width=0.46\linewidth]{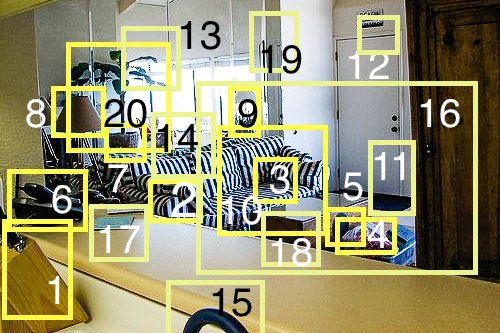}
\includegraphics[width=0.053\linewidth]{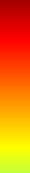}
\caption{Left: an image with a target class {\it cat}. The weak detector identified two box proposals with high scores. The best strategy in this case is to do a series of box verifications. Right: an image with a target class {\it potted plant}. The weak detector identified many box proposals with low scores. The best strategy is to draw a box.}
  \vspace{-0.2cm}
\label{fig:intuition}
\end{figure}

The success of box verification series depends on a variety of factors. 
For example, large objects on homogeneous backgrounds are likely to be found early in the series, and hence require little annotation time (Fig.~\ref{fig:intuition}, left).
However, small objects in crowded scenes might require many iterations, or could even not be found at all  (Fig.~\ref{fig:intuition}, right).
Furthermore, the stronger the detector is, the more likely it is to correctly localize new objects, and to do so early in the series.
Finally, the higher the desired box quality (i.e. how tight they should be), the lower the rate of positively verified boxes. This causes longer series, costing more annotation time. 
Therefore, in some situations manual box drawing~\cite{papadopoulos17iccv,su12aaai} is preferable. While more expensive than one verification, it always produces a box annotation.
When an annotation episode consists of many verifications, its duration can be longer than the time to draw a box, depending on the relative costs of the two actions.
Thus, different forms of annotation are more efficient in different situations.

In this paper we introduce Intelligent Annotation Dialogs (\IAD{}) for bounding box annotation.
Given an image, detector, and target class to be annotated, the aim of \IAD{} is to automatically choose the sequence of annotation actions that results in producing a bounding box in the least amount of time.
We train an \IAD{} agent to select the type of action based on previous experience in annotating images.
Our method automatically adapts to the difficulty of the image, the strength of the detector, the desired quality of the boxes, and other factors.
This is achieved by modeling the episode duration as a function of problem properties. We consider two alternative ways to do this, either 
\begin{enumerate*}[label={\alph*)}]
\item by predicting whether a proposed box will be positively or negatively verified (Sec.~\ref{sec:approach1}), or 
\item by directly predicting the episode duration (Sec.~\ref{sec:approach2}).
\end{enumerate*}

We evaluate \IAD{} by annotating bounding boxes in the PASCAL VOC 2007 dataset~\cite{everingham10ijcv} in several scenarios: \begin{enumerate*}[label={\alph*)}]
\item with various desired quality levels;
\item with detectors of varying strength; and
\item with two ways to draw bounding boxes, including a recent method which only takes 7s per box~\cite{papadopoulos17iccv}.
\end{enumerate*}
In all scenarios our experiments demonstrate that thanks to its adaptive behavior \IAD{} speeds up box annotation compared to manual box drawing alone, or box verification series alone. 
Moreover, it outperforms any fixed combination of them in most scenarios.
Finally, we demonstrate that \IAD{} learns useful strategies in a complex realistic scenario where the detector is continuously improved with the growing amount of the training data. 
Our IAD code is made publicly available\footnote{\ssmall{\url{https://github.com/google/intelligent_annotation_dialogs}}}.

\section{Related work }
\label{sec:related}

\mypartop{Drawing bounding boxes} 

Fully supervised object detectors are trained on data with manually annotated bounding boxes, which is costly. 
The reference box drawing interface~\cite{su12aaai} used to annotate ILSVRC~\cite{russakovsky15ijcv} requires 25.5s for drawing one box. 
Recently, a more efficient interface reduces costs to 7.0s without compromising on quality~\cite{papadopoulos17iccv}.
We consider both interfaces in this paper.

\mypar{Training object detectors from image-level labels}
Weakly Supervised Object Localization (WSOL)~\cite{bilen16cvpr,cinbis14cvpr,deselaers10eccv,haussmann17cvpr,kantorov16eccv,zhu17iccv} methods train object detectors from images labeled only as containing certain object classes, but without bounding boxes. This avoids the cost of box annotation, but leads to considerably weaker detectors than their fully supervised counterparts~\cite{bilen16cvpr,haussmann17cvpr,kantorov16eccv,zhu17iccv}. To produce better object detectors, extra human annotation is required.

\mypar{Other forms of weak supervision} 
Several works aim to reduce annotation cost of manual drawing by
using approximate forms of annotation. 
Eye-tracking is used for training object detectors~\cite{papadopoulos14eccv} and for action recognition in videos~\cite{mathe12eccv}.  
Point-clicks are used to derive bounding box annotations in images~\cite{papadopoulos17cvpr} and video~\cite{mettes16eccv}, and to train semantic segmentation models~\cite{bearman16eccv,bell15cvpr,wang14cviu}.
Other works train semantic segmentation models using scribbles~\cite{lin16cvpr,xu15cvpr}.

In this paper we build on box verification series~\cite{papadopoulos16cvpr}, where boxes are iteratively proposed by an object detector and verified by a human annotator. Experiments show that humans can perform box verification reliably (Fig.~\num{6} of~\cite{papadopoulos16cvpr}).
Besides, the Open Images dataset~\cite{openimages} contains \num{2.5} Million boxes annotated in this manner, demonstrating it can be done at scale.

\mypar{Interactive annotation}
Several works use human-machine collaboration to efficiently produce annotations. 
These works address interactive segmentation~\cite{boykov01iccv,Rother04-tdfixed,castrejon17cvpr,jain13iccv,jain16hcomp,nagaraja15iccv}, attribute-based fine-grained image classification~\cite{branson10eccv,parkash12eccv,biswas13cvpr,wah14cvpr}, and interactive video annotation~\cite{vondrick13ijcv}.
Branson \etal~\cite{branson14cvpr} transform different types of location information (e.g.~parts, bounding boxes, segmentations) into each other with corrections from an annotator.
These works follow a predefined annotation protocol, whereas we explore algorithms that can automatically select questions, adapting to the input image, the desired quality of the annotation, and other factors.

The closest work~\cite{russakovsky15cvpr} to ours proposes human-machine collaboration for bounding box annotation. Given a repertoire of questions, the problem is modeled with a Markov decision process. Our work differs in several respects. \begin{enumerate*}[label={(\arabic*)}] 

\item While Russakovsky \etal~\cite{russakovsky15cvpr} optimizes the expected precision of annotations over the whole dataset, our method delivers quality guarantees on each individual box.

\item Our approach of Sec~\ref{sec:approach1} is mediated by predicting the probability of a box to be accepted by an annotator. Based on this, we provide a provably optimal strategy which minimizes the expected annotation time.

\item Our reinforcement learning approach of Sec.~\ref{sec:approach2} learns a direct mapping from from measurable properties to annotation time, while avoiding any explicit modelling of the task.

\item Finally, we address a scenario where the detector is iteratively updated (Sec.~\ref{sec:experiment2}), as opposed to keeping it fixed.
\end{enumerate*}

\mypar{Active learning (AL)}
In active learning the goal is to train a model while asking human annotations for unlabeled examples which are expected to improve the model accuracy the most. 
It is used in computer vision to train whole-image
classifiers~\cite{joshi09cvpr,kovashka11iccv}, object class detectors~\cite{vijayanarasimhan14ijcv,yao2012cvpr}, and semantic segmentation~\cite{siddiquie10cvpr,vijayanarasimhan08nips,vijayanarasimhan09cvpr}.
While the goal of AL is to select a subset of the data to be annotated, this paper aims at minimizing the time to annotate each of the examples.

\mypar{Reinforcement learning} 
Reinforcement learning (RL) traditionally aims at learning policies that allow autonomous agents to act in interactive environments. 
Reinforcement learning has a long tradition e.g.~in robotics~\cite{asada96ml,michels05icml,levine16jmlr}. 
In computer vision, it has mainly been used for active vision tasks~\cite{caicedo15iccv,bellver16nipsws,jayaraman17arxiv}, such as learning a policy for the spatial exploration of large images or panoramic images. 
Our use of RL differs from this, as we learn a policy for image annotation, not for image analysis. 
The learned policy enables the system to dynamically choose the annotation mechanism by which to interact with the user. 
\section{Problem definition and motivation}
\label{sec:methods}

\subsection{Why use intelligent annotation dialogs?}
\label{sec:intuitionandmotivation}

In this paper we tackle the problem of producing bounding box annotations for a set of images with image-level labels indicating which object classes they contain.
Consider annotating a {\em cat} in Fig.~\ref{fig:intuition} (left).
The figure shows two bounding boxes found by the detector.
We notice that:
\begin{enumerate*}[label={\alph*)}]
\item the image is relatively simple with only one distinct object;
\item there are only few high-scored {\it cat} detections;
\item they are big;
\item we might know a-priori that the detector is strong for the class {\it cat} and thus detections for it are often correct.
\end{enumerate*}
As box verification is much faster than drawing, the most efficient way to annotate a box in this situation is with a box verification series.

Now consider instead annotating a {\it potted plant} in Fig.~\ref{fig:intuition} (right). 
We notice that:
\begin{enumerate*}[label={\alph*)}]
\item the image is cluttered with many details;
\item there are many low-scored {\it potted plant} detections;
\item they are small;
\item we might know a-priori that the detector is weak for this class and thus the detections for it are often wrong.
\end{enumerate*}
In this situation, it is unlikely that the correct bounding box comes early in the series.
Thus, manual box drawing is likely to be the fastest annotation strategy.

Even during annotation of one image-class pair, the best strategy may combine both annotation types:
Given only one high-scored box for {\it cat}, the best expected strategy is to verify
one box, and, if rejected, ask manual box drawing.

These examples illustrate that every image, class and detector output requires a separate treatment for designing the best annotation strategy. 
Thus, there is need for a method that can take advantage of this information to select the most time efficient  sequence of annotation  actions.
In this paper, we propose two methods to achieve this with the help of Intelligent Annotation Dialog (\IAD{}) agents. 
In our first approach (Sec.~\ref{sec:approach1}) we explicitly model the expected episode duration by taking into consideration the probability for each proposed box to be accepted.
Our second approach (Sec.~\ref{sec:approach2}) casts the problem in terms of reinforcement learning and leans a strategy from trial-and-error interactions without an intermediate modeling step.

\begin{figure}[t]
\includegraphics[width=\linewidth]{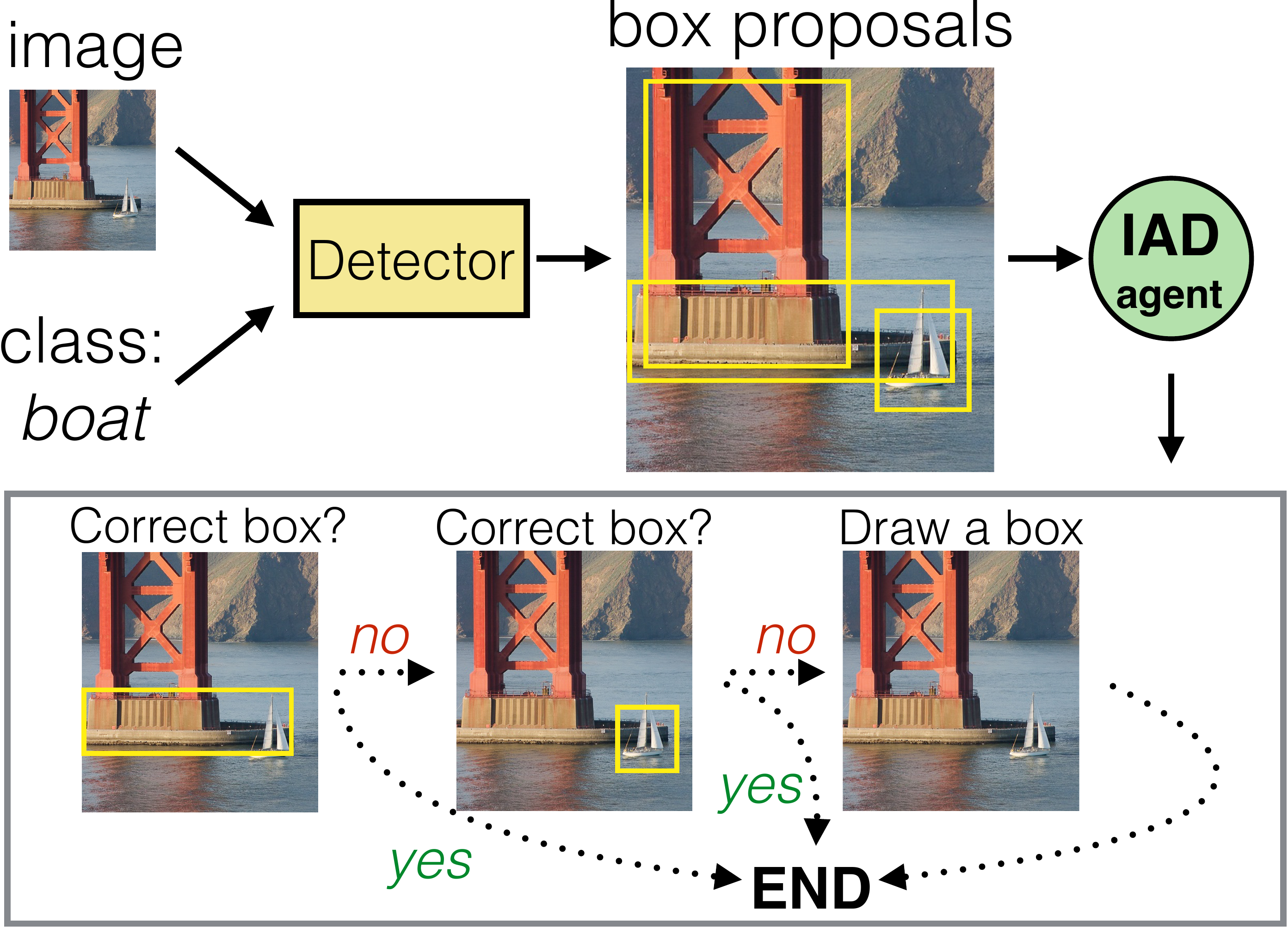}
\caption{Intelligent Annotation Dialog agent in action. For a given image and class {\em boat} the detector identifies a set of box proposals. IAD agent produces a planned dialog $V^2D$ that means that the first two box proposals are verified and if none of them is accepted, manual box drawing is done. In reality, the annotation terminates after two box verifications.}
\label{fig:workflow-boat}
\vspace{-4mm}
\end{figure}

\subsection{Problem definition}
\label{sec:probdef}

We are given an image with image-level labels that indicate which object classes it contains.  
We treat each class independently, and we want to produce one bounding box given a single image-class pair.
In particular, given that the image contains a set $\mathcal{B}^{\star}$ of object instances of the target
class, we want to produce a bounding box $\hat{b}$ of sufficient quality around one such object
$b^\star$. 
We measure the quality in terms of Intersection-over-Union (IoU) and we want to find $\hat{b}$
such that there exists $b^\star \in \mathcal{B}^\star: \text{IoU}(\hat{b},b^\star)\ge \alpha$.
More specifically, we want to automatically construct a sequence of actions which produces $\hat{b}$
while minimizing annotation time, choosing from two annotation
actions: manual bounding box drawing $D$~\cite{papadopoulos17iccv,su12aaai} that takes $t_D$ seconds
and bounding box verification $V$~\cite{papadopoulos16cvpr} that takes $t_V$ seconds.

We design the annotation dialog to end with a successfully annotated bounding box.  Logically, the only
possible planned sequence of actions which does this has the form $V^mD$. 
No sequence of verification $V$ is guaranteed to produce a bounding box, so if $m$ verifications fail to produce one, manual drawing
is required. Conversely, manual drawing always produces a box and the dialog ends.
Fig~\ref{fig:workflow-boat} illustrates how \IAD{} agent produces a planned sequence of $V^2D$ for the task of detecting a {\it boat} in the image with several detections.
In reality, only a sequence of actions $V^2$ is executed because a boat is found at the second verification.

Verification questions are generated using an object detector. Papadopoulos \etal~\cite{papadopoulos16cvpr}
present the highest scored detection to the annotator. Upon rejection, they remove boxes which
highly overlap with the rejected one (this procedure is called \emph{search space reduction}), after which they
present the next box with the highest score. In this paper we assume the detector stays constant during a
single annotation dialog, which means we can do search space reduction by non-maximum
suppression (NMS).
Let us denote by $B_0$ the sequence of detections followed by NMS. 
Because of NMS, we can assume boxes in $B_0$ to be independent for verification. 
Let $\mathcal{S}$ be the set of all possible sequences of distinct elements in $B_0$. Now our goal is to plan a
sequence of actions $\pi = V^mD$ on a sequence $S_m = (s_1, \ldots, s_m) \in \mathcal{S}$.

We can now formally define the optimization criterion for the \IAD{} agent.
Let $t(V^mD,S_m)$ be the duration of the episode when strategy $V^mD$ is applied to a sequence $S_m$ and let us denote its expected duration as $T(V^mD,S_m)$.
The task of \IAD{} is to choose
1) the maximum number of verifications $m=k$ that will define a sequence of actions $V^{k}D$, and
2) a sequence of boxes $A_k=(a_1,\ldots,a_k) \in S$ such that the duration of the episode is minimized in expectation:
\begin{align}
  \begin{split}
    T(V^kD,A_k) \le T(V^mD,S_m), \\
m \in \{0,\ldots,n\}, \forall S_m \in \mathcal{S}.
  \end{split}  
  \label{eq:startingcondition}
\end{align}
\section{Methods}

We now present our two methods to construct Interactive Annotation Dialogs (\IAD{}).

\subsection{\IAD{} by predicting probability of acceptance}
\label{sec:approach1}

One way to minimize the expected duration of the episode is by estimating the probability that the proposed boxes will be accepted by the annotator.
We can train a classifier $g$ that will predict if the box $b_i \in B_0$ is going to be accepted or not as a function of various parameters of the state of the episode.
By looking at the probability of acceptance $p(b_i)$ for every box, we can compute the expected duration of the episode $T(V^mD,S_m)$ for any $V^mD$ and $S_m$.
Given this acceptance probability estimation, we show that there exists a simple decision rule that chooses $m$ and $S_m$ so as to minimize the expected episode duration.

\mypar{Optimal strategy}

Suppose for now that we know the probabilities $p(b_i)$ for every box $b_i$ to be accepted at a quality level $\alpha$:
\begin{align}
p(b_i)=\mathbb{P}[\max_{b^\star \in B^{\star}}\{\text{IoU}(b_i,b^{\star})\} \ge \alpha].
\end{align}
Later in this section we will explain how to estimate $p(b_i)$ in practice.

Imagine for a moment that we have only one box proposal $b_1$.
In this case the only two possible sequences of actions are $D$ and $V^1D$. 
Let us compute the expected time until the end of the episode for both of them.
The episode duration for strategy $D$ is just the time required for manual drawing: $T(D) = t_D.$

For the second strategy $V^1D$, the end of the episode is reached with probability $p(b_1)$ when a box proposal is accepted and with probability $q(b_1)= 1-p(b_1)$ when manual drawing is done.
Hence, the expected duration of the episode is 
\begin{align}
T(V^1D,(b_1)) = t_V+q(b_1)t_D.
\end{align}
As we want to choose the strategy with the lowest expected duration of the episode, $D$ is preferred to $V^1D$ if~~$T(D) \le T(V^1D,(b_1))$, i.e.
\begin{align}
t_D \le t_V+q(b_1)t_D \iff p(b_1) \le t_V/t_D.
\label{eq:rule}
\end{align}
Now let us go back to a situation with a sequence of box proposals $B_0$.
We sort $B_0$ in the order of decreasing probability of acceptance $p(b_i)$, resulting in a sequence of boxes $\bar{S}_n$.
Consider the following strategy (Alg.~\approachClassification{}): Verify boxes from $\bar{S}_n$ for which $p(b_i)> t_V/t_D$; if none of them is accepted, then do manual box drawing. 
We claim that the strategy produced by \approachClassification{} is optimal, i.e. it minimizes the expected duration of the episode.

\begin{algorithm}[h]
   \caption*{{\bf Algorithm} \approachClassification}
   \label{alg:decision-making}
\begin{algorithmic}[1]
	\STATE {\bfseries Input:} $B_0 = (b_1,\ldots,b_n); p(b_1),\ldots,p(b_n); t_V; t_D$
	\STATE $\bar{S}_n = (\bar{s}_1, \bar{s}_2,\ldots,\bar{s}_n) \gets \text{sort}(B_0)$ by $p(b_i)$
	\STATE $\pi = ()$
	\STATE $A_k = ()$
	\WHILE {$p(\bar{s}_i) > t_V/t_D$}
		\STATE $A_k \gets A_k \frown \bar{s}_i$
	\ENDWHILE
	\STATE {$\pi \gets V^kD$}
	\RETURN {sequence of actions $\pi$, sequence of boxes $A_k$}
\end{algorithmic}
\end{algorithm}

\begin{theorem}
If probabilities of acceptance $\{p(b_i)\}$ are known, the strategy of applying a sequence of actions $V^kD$ defined by \approachClassification{} to a sequence of boxes $A_k$
minimizes the annotation time, i.e.\ for all $m\in \{0,\ldots,n\}$ and for all box sequences $S_m$:
\begin{align}
	T(V^kD,A_k) \le T(V^mD,S_m)
\end{align}
\end{theorem}

\begin{proof}[Sketch of the proof] 

The proof consists of two parts. First, we show that for any strategy $V^mD$, the best 
box sequence is obtained by sorting the available boxes by their probability of 
acceptance and using the first $m$ of them. Second, we show that the number
of verification steps found by \approachClassification{}, $k$, is indeed the optimal one.

We start by rewriting the expected episode length in closed form. 
For a strategy $V^m D$ and any sequence of boxes, 
$S_m=(s_1,\dots,s_m)$, we obtain
\begin{small}
\begin{align}
&T(V^mD, S_m ) = t_{V}+q(s_1)t_V+q(s_1)q(s_2)t_V + \dots  \notag\\
&\quad + q(s_1)q(s_2)\cdots q(s_{m-1})t_V + q(s_1)q(s_2)\cdots q(s_m)t_D \notag
\\
&\quad= t_V\sum_{l=0}^{m-1}\prod_{j=1}^l q(s_j)  + t_D\prod_{j=1}^m q(s_j).
\label{eq:expectation_unrolled}
\end{align}
\end{small}
Our first observation is that~\eqref{eq:expectation_unrolled} is monotonically decreasing as a function of $q(s_1),\dots,q(s_m)$. Consequently, the smallest value is obtained by selecting the set of $m$ boxes that have the smallest rejection probabilities.
To prove that their optimal order is sorted in decreasing order, assume that $S_m$ is not sorted, i.e.\ there exists an index $l\in\{1,\dots,m-1\}$ for which $q(s_l)>q(s_{l+1})$. 
We compare the expected episode length of $S_m$ to that of a sequence $\tilde S_m$ in which $s_{l}$ and $s_{l+1}$ are at switched positions. Using~\eqref{eq:expectation_unrolled} and noticing that many of the terms cancel out, we obtain
\begin{small}
\begin{align}
T&(V^mD, S_m )- T(V^mD, \tilde S_m ) \notag  \\
&= t_V\big(q(s_l)\!-\!q(s_{l+1})\big)\Big(\prod_{j=1}^{l-1} q_j\Big) \quad>0.
\end{align}
\end{small}
This shows that $\tilde S_m$ has strictly smaller expected episode length than $S_m$, so $S_m$ cannot have been the optimal order. 

Consequently, for any strategy $V^mD$, the optimal sequence is to sort the boxes by decreasing probability of rejection, \ie increasing acceptance probability.
We denote it by $\bar S_m=(\bar s_1,\dots,\bar s_m)$. 

Next, we show that the number, $k$, of verification actions found by the \approachClassification{} algorithm is optimal, \ie $V^k D$ is better or equal to $V^mD$ for any $m \ne k$.
As we already know that the optimal box sequence for any strategy $V^m D$ is $\bar S_m$, it is enough to show that
\begin{small}
  \begin{align}
    T(V^{m-1}D, \bar S_{m-1}) 
    &\geq T(V^mD, \bar S_{m}),
    \label{eq:E_decreasing}
 \end{align}
\end{small}
for all $m\in\{1,\dots,k\}$, and
\begin{small}
  \begin{align}
    T(V^{m-1} D, \bar S_{m-1}) 
    &\leq 
    T(V^{m}D,\bar S_m).
    \label{eq:E_increasing}
    \end{align}
\end{small}
for all $m\in\{k+1,\dots,n\}$. 
To prove these inequalities, we again make use of expression~\eqref{eq:expectation_unrolled}. For any  $m\in\{1,\dots,n-1\}$ we obtain
\begin{small}
\begin{align}
&T(V^{m} D, \bar S_{m}) - T(V^{m-1}D,\bar S_{m-1})
\notag \\
&= t_V\prod_{j=1}^{m-1} q(\bar s_j) + t_D\prod_{j=1}^{m} q(\bar s_j) - t_D\prod_{j=1}^{m-1} q(\bar s_j)
\notag \\
&= \Big(\prod_{j=1}^{m-1} q(\bar s_j)\Big)\big(t_V + q(\bar s_m)t_D - t_D\big).
\label{eq:Vm_Vm1}
\end{align}
\end{small}
For $m\in\{1,\dots,k\}$, we know that $p(\bar{s}_m) > t_V/t_D$ by construction of the strategy.
As in~\eqref{eq:rule}, this is equivalent to $t_V + q(\bar s_m)t_D - t_D \geq 0$.
Consequently, \eqref{eq:Vm_Vm1} is non-negative in this case, and inequality \eqref{eq:E_decreasing} is confirmed.
For $m\in\{k+1,\dots,n\}$, we know $p(\bar{s}_m) \leq t_V/t_D$, again by construction. 
Consequently, $t_V + q(\bar s_m)t_D - t_D \leq 0$, which shows that \eqref{eq:Vm_Vm1} is nonpositive in this case, confirming \eqref{eq:E_increasing}.
\end{proof}

\mypar{Predicting acceptance probability}

To follow the optimal strategy \approachClassification{}, we need the probabilities of acceptance
$\{p(b_i)\}$ which we estimate using a classifier $g$. To obtain these probabilities we start with a (small) set $Z_0$
of annotated bounding boxes on a set of images $I_0$. We apply a detector $f_0$ on $I_0$ to obtain a
set of detections $B_0$. Afterwards, we generate a feature vector $\phi_i$ for every box $b_i \in
B_0$. The exact features are specified in Sec.~\ref{sec:experimentalsetup} and include measurements
such as detector scores, entropy, and box-size.

Next, we simulate verification responses for box proposals $B_0$ of every image-class pair with known ground truth. 
A box $b_i$ gets label $y_i = 1$ if its IoU with any of the ground truth boxes is great or equal to $\alpha$, otherwise it gets label $0$.
This procedure results in feature-label pairs $(\phi_i, y_i)$ that serve as a dataset for training a probabilistic classifier $g$. 

Intuitively, the classifier learns that, for example, boxes with high detector's score are more
likely to be accepted than boxes with low detector's score, bounding boxes for class {\it cat} are
more likely to be accepted than bounding boxes for class {\it potted plant}, and smaller bounding boxes are less likely to be accepted than big ones.

\subsection{\IAD{} by reinforcement learning}
\label{sec:approach2}

The problem of finding a sequence of actions to produce a box annotation can be naturally formulated as a reinforcement learning problem. 
This approach allows us to learn a strategy directly from trial-and-error experience and to avoid the explicit modeling of Sec.~\ref{sec:approach1}.
To construct an optimal strategy it does not need any prior knowledge about the environment.
Thus, it is easily extensible to other types of actions or to stochastic environments with variable response time by an annotator.

Suppose that bounding boxes in an episode are verified in order of decreasing detector's score given by $B_0$.
In an {\it episode} of annotating one image for a given target class, the \IAD{} agent interacts with the {\it environment} in the form of the annotator.
A {\it state} $s_\tau$ is characterised by the properties of a current image, detector and a current box proposal (as $\phi_i$ in Sec.~\ref{sec:approach1}).
In each state the agent has a choice of two possible {\it actions} $a$:
\begin{enumerate*}[label={\arabic*)}]
\item ask for verification of the current box ($a=V$) and 
\item ask for a manual drawing ($a=D$)
\end{enumerate*}
The {\it reward} at every step $\tau$ is the negative time required for the chosen action: $r_\tau = -t_V$ and $r_\tau = -t_D$.
If a box is positively verified or manually drawn, the episode terminates with a reward \num{0}. 
Otherwise the agent finds itself in the next state corresponding to the next highest-scored box proposal in $B_0$.
The total {\it return} of the episode is the sum of rewards over all steps.
Denoting the number of steps after which an episodes terminates by $K$, the return is $R = \sum_{\tau=1}^K r_\tau$.
This is equal to $-(K-1)t_V-t_D$ if the episode finished with manual drawing, or $-Kt_V$ if it finished with box acceptance. 
By trying to maximise the return $R$, the agent learns a policy $\pi$ that minimises the total annotation time.
This results in a strategy that consists of a sequence of actions $\pi$ applied to a sequence of boxes $B_0$.

\mypar{Training the agent}

The agent can learn the optimal policy $\pi$ from trial and error interactions with the environment. 
As in Sec.~\ref{sec:approach1}, we train on a small subset of annotated bounding boxes $Z_0$.
We learn a policy with {\it Q-learning} which learns to approximate {\it Q-function}
$Q^\pi(a,s_\tau)$ that indicates what return the agent should expect at state $s_\tau$ after taking an action $a$ and after that following a strategy $\pi$.

\section{Experiments}
\label{sec:experiments}

\subsection{Experimental setup}
\label{sec:experimentalsetup}

We evaluate the performance of the \IAD{} approach on the task of annotating bounding boxes on the PASCAL VOC 2007 trainval dataset.
In all experiments our detector is Faster-RCNN~\cite{ren15nips} using Inception-ResNet~\cite{szegedy16cvpr} as base network.

\mypar{Annotator actions and timings}
We simulate the annotator based on the ground truth bounding boxes.
When asked for verification, a simulated human annotator deterministically accepts a box proposal if IoU $\geq \alpha$ and it takes $t_V = 1.8$ seconds~\cite{papadopoulos16cvpr}. 
When the simulated annotator is asked to draw a box, we use the ground truth box.
We consider two interfaces for drawing: the classical manual drawing $M$~\cite{su12aaai} and the new faster Extreme Clicking $X$~\cite{papadopoulos17iccv}.
We consider that it takes a simulated user $t_M = 25.5$ or $t_X = 7$ seconds to return a bounding box that corresponds to any of the objects $b^\star$~\cite{su12aaai, papadopoulos17iccv}.

\mypar{Box proposal order}

The order of box proposals for verifications is set to be $B_0$, i.e. in decreasing order of detector's score (Sec.~\ref{sec:probdef}).
Then, the optimality condition of strategy \approachClassification{} assumes that a box with higher score is more likely to be accepted than a box with lower score.
Empirically, we observe only rare cases when this assumption is violated, but even then, changing the order does not improve results. 
Thus, we keep the original order $B_0$ for computational efficiency and consistency with \approachRL{}.
The images come in the same fixed random order for all methods.

\mypar{Box features}

When predicting the acceptance probability (Sec.~\ref{sec:approach1}) and during reinforcement learning (Sec.~\ref{sec:approach2}), we use the following features $\phi_i$ characterizing box $b_i$, image, detector, and target class: 
\begin{enumerate*}[label={\alph*)}]
\item prediction score of the detector on the box: $d(b_i)$;
\item relative size of the box $b_i$ in the image;
\item average prediction score of all box proposals for the target class;
\item difference between c) and $d(b_i)$;
\item difference between the maximum score for the target class among all box proposals and $d(b_i)$;
\item one-hot encoding of class.
\end{enumerate*}

\mypar{IAD-Prob}
To predict box acceptance probabilities,
we use a neural network classifier with \num{2} to \num{5} layers containing \num{5} to \num{50} neurons in each layer for predicting the acceptance of a box and these parameters are chosen in cross-validation (Sec.~\ref{sec:approach1}).
We experimented with other types of classifiers including logistic regression and random forest and did not find any significant difference in their performance.

\mypar{IAD-RL}
We learn a policy for the reinforcement learning agent with a method similar to
\citep{mnih13nips} (Sec.~\ref{sec:approach2}). 
The function approximation of Q-values is a fully-connected neural network with \num{2} layers and \num{30} neurons at every layer.
We learn it from interactions with simulated environment using experience replay.
We use exploration rate $\epsilon=0.2$, mini-batches of size \num{64} and between \num{500} and \num{1000} training iterations.
A subset of training samples is reserved for validation: we use it for choosing parameter of neural network and for early stopping.

\subsection{\IAD{} with a fixed detector}
\label{sec:experiment1}

\mypartop{Scenarios}
We evaluate our methods in several scenarios, by varying the following properties of the problem:
\begin{enumerate*}[label={\alph*)}] 
\item the desired quality of boxes,
\item the strength of the detector, and
\item which interface is used to draw a box.
\end{enumerate*}
Intuitively, different properties tend to prioritize different actions $V$ or $D$.
The higher the desired quality is, the more frequently manual box drawing is needed.
When the detector is strong, box verification is successful more often and is preferred to drawing due to its small cost.
Finally, using the fast Extreme Clicking interface, manual drawing is cheaper and becomes more attractive.
Specifically, we consider the following three configurations, each for both quality levels:
\begin{enumerate}
  \descitem{Weak detector, slow drawing, varying quality}
  Classical, slow interface to draw boxes~\cite{su12aaai} with a weak detector. 
  To train the detector (Sec.~\ref{sec:experimentalsetup}), we first produce bounding box estimates using standard Multiple Instance Learning (MIL, e.g.~\cite{bilen14bmvc,cinbis14cvpr,siva11iccv}). 
  The first two columns of Tab.~\ref{tab:exp1} report the average time per one annotation episode.
\vspace{-0.1cm}
  \descitem{Weak detector, fast drawing, varying quality}
  The fast Extreme Clicking~\cite{papadopoulos17iccv} for drawing boxes. 
  We report the results in columns \num{3} and \num{4} of Tab.~\ref{tab:exp1}.
\vspace{-0.1cm}
  \descitem{Strong detector, fast drawing, varying quality}
  In many situations we have access to a reasonably strong detector before starting annotation of a new dataset. 
  To model this we train $f_0$ on the PASCAL 2012 dataset train set which contains 16k boxes.
  The results are presented in the last two columns of Tab.~\ref{tab:exp1}.
\end{enumerate}
\vspace{-0.3cm}

\mypar{Dataset}

We use PASCAL 2007 trainval~\cite{everingham10ijcv}, where we assume that image-level annotations are available for all images, whereas bounding boxes are given only in a small subset of images $Z_0$.
The task is to annotate the rest of the images $Z'$ with bounding boxes.
$Z$ and $Z'$ are set with \num{10}-fold validation and the reported results are averages over them.

\mypar{Standard strategies}
As baselines, we consider two standard annotation strategies. The first is to always do manual drawing ($D$). 
The second is to run box verification series, followed by drawing if all available boxes have been rejected ($V^*D$).
This strategy is guaranteed to terminate successfully while being the closest to~\cite{papadopoulos16cvpr}.

\begin{table*}[t]  
  \vspace{-.4cm}
  \centering
  \begin{tabular}{l|rr|rr|rr}
    \toprule
    {\bf Drawing technique}	& \multicolumn{2}{c}{Slow drawing} & \multicolumn{4}{c}{Fast drawing} \\
    {\bf Detector}	& \multicolumn{2}{c}{Weak detector} & \multicolumn{2}{c}{Weak detector} & \multicolumn{2}{c}{Strong detector} \\
    {\bf Quality level}	& $\alpha=0.7$ & $\alpha=0.5$ & $\alpha=0.7$ & $\alpha=0.5$ & $\alpha=0.7$ & $\alpha=0.5$ \\
    \midrule
    {$D$ (standard)} & \num{25.50} $\pm$ \num{0.00} & \num{25.50} $\pm$ \num{0.00} & \textbf{\num{7.00}} $\pm$ \num{0.00} & \num{7.00} $\pm$ \num{0.00} & \num{7.00} $\pm$ \num{0.00} & \num{7.00} $\pm$ \num{0.00} \\
    {$V^1D$} & \cellcolor{yellow!40}\textbf{\num{23.01} $\pm$ \num{0.07}} & \num{17.30} $\pm$ \num{0.07} & \num{7.62} $\pm$ \num{0.02} & \textbf{\num{6.05}} $\pm$ \num{0.02} & \cellcolor{yellow!40}\textbf{\num{3.45} $\pm$ \num{0.01}} & \num{2.50} $\pm$ \num{0.01} \\
    {$V^2D$} & \num{23.79} $\pm$ \num{0.06} & \num{16.67} $\pm$ \num{0.06} & \num{8.92} $\pm$ \num{0.02} & \num{6.67} $\pm$ \num{0.02} & \cellcolor{yellow!40}\num{3.48} $\pm$ \num{0.01} & \cellcolor{yellow!50}\textbf{\num{2.45}} $\pm$ \num{0.01} \\
    {$V^3D$} & \num{24.67} $\pm$ \num{0.07} & \textbf{\num{16.38}} $\pm$ \num{0.07} & \num{10.21} $\pm$ \num{0.02} & \num{7.32} $\pm$ \num{0.03} & \num{3.65} $\pm$ \num{0.02} & \num{2.48} $\pm$ \num{0.01}\\
    {$V^\star D$ (standard)} & \num{42.29} $\pm$ \num{0.07} & \num{17.37} $\pm$ \num{0.07} & \num{31.82} $\pm$ \num{0.11} & \num{11.46} $\pm$ \num{0.04} & \num{8.83} $\pm$ \num{0.09} & \num{3.18} $\pm$ \num{0.02} \\
    \midrule
    {\approachClassification{}} & \cellcolor{yellow!40}\num{23.07} $\pm$ \num{0.23} & \cellcolor{yellow!40}\num{12.64} $\pm$ \num{1.29} & \cellcolor{yellow!40}\num{6.81} $\pm$ \num{0.02} & \cellcolor{yellow!40}\num{5.86} $\pm$ \num{0.04} & \cellcolor{yellow!40}\num{3.42} $\pm$ \num{0.18} & \num{2.73} $\pm$ \num{0.08} \\
    {\approachRL{}} & \num{23.62} $\pm$ \num{0.38} & \num{16.30} $\pm$ \num{0.09} & \cellcolor{yellow!40}\num{6.83} $\pm$ \num{0.03} & \cellcolor{yellow!40}\num{5.89} $\pm$ \num{0.05} & \cellcolor{yellow!40}\num{3.60} $\pm$ \num{0.07} & \num{2.66} $\pm$ \num{0.06} \\
	\midrule
	lower bound & \num{18.55} $\pm$ \num{0.05} & \num{10.23} $\pm$ \num{0.04} & \num{5.99} $\pm$ \num{0.01} & \num{4.66} $\pm$ \num{0.01} & \num{2.80} $\pm$ \num{0.01} & \num{2.19} $\pm$ \num{0.01} \\
    \bottomrule
  \end{tabular}
  
  \caption{Average episode duration for {\em standard}, {\em fixed} and \IAD{} strategies in scenarios varying in drawing speed, strength of detector and quality level. \textbf{Best fixed strategy} results are highlighted in bold. The \fcolorbox{yellow}{yellow!40}{best result} of each scenario is indicated in yellow (multiple highlights if very close). The two IAD agents do approximately equally well.}
  \vspace{-.4cm}
\label{tab:exp1}
\end{table*}

\mypar{Fixed strategies}
We introduce a family of fixed strategies that combine the two actions $V$ and $D$ in a predefined manner, without adapting to a particular image, class and detector: 
$V^1D$,
$V^2D$, and
$V^3D$.

\mypar{Lower bound}
We also report the lower bound on the duration of the annotation episode.
If we knew which box (Sec~\ref{sec:probdef}) in the proposal sequence $B_0$ is the first that will be accepted, we could choose a sequence of actions that leads to the lowest annotation cost.
If accepted box is at the position $k^\star$ in sequence $B_0$, then the strategy is the following.
If the cost of $k^\star$ verifications is lower than the cost of a drawing, then the verification series is done, otherwise, drawing is done.
Note how this lower bound requires knowing the ground-truth bounding box. 
So, it is only intended to reveal the limits of what can be achieved by the type of strategies that we explore.

\mypar{Results}

Tab.~\ref{tab:exp1} shows that the scenario settings indeed influence the choice between $V$ and $D$, along three dimensions:
\begin{enumerate*}[label={\alph*)}]
\item When annotations of higher quality are required, the best fixed strategy does fewer verifications, i.e. it resorts to manual drawing ealier in the series than when lower quality is acceptable (columns \num{1} vs. \num{2}, \num{3} vs. \num{4}, \num{5} vs. \num{6}). 
\item When the detector is strong (columns \num{5} and \num{6}), the best fixed strategy does more box verifications than with a weak detector (columns \num{3} and \num{4}).
\item When manual drawing is fast (columns \num{3} and \num{4}), the best fixed strategy tends to do fewer box verifications than when drawing is slow (columns \num{1} and \num{2}).
\end{enumerate*}
The gap to the lower bound indicates how hard each of the scenarios is.

Importantly, both of our \IAD{} strategies outperform any standard strategy in all scenarios.
Moreover, IAD-Prob is significantly better than the best fixed strategy in three scenarios, equal in two, and worse in one.
No single fixed strategy works well in {\em all} scenarios, and finding the best fixed strategy requires manual experimentation.
In contrast, IAD offers a principled way to automatically construct an adaptive strategy that works well in {\em all} problem settings.
Indeed, the consistent competitive performance of \IAD{} demonstrates that it learns to adapt to the scenario at hand. 

\subsection{\IAD{} with an iteratively improving detector}
\label{sec:experiment2}

\begin{figure*}[th]
\includegraphics[width=0.3\linewidth]{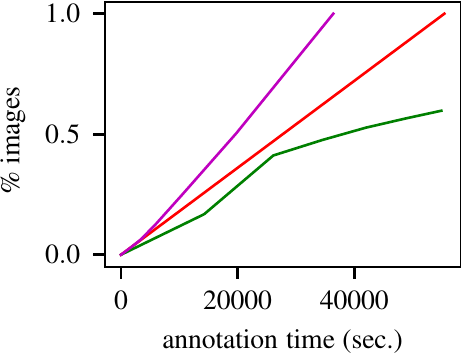}
\hspace{0.5cm}
\includegraphics[width=0.3\linewidth]{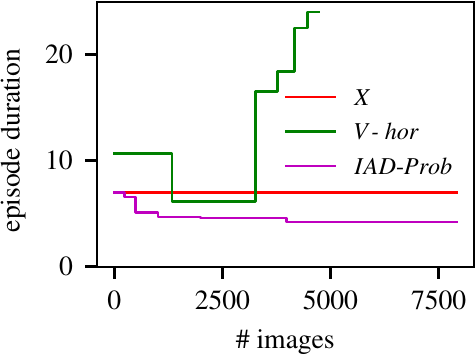}
\hspace{0.5cm}
\includegraphics[width=0.3\linewidth]{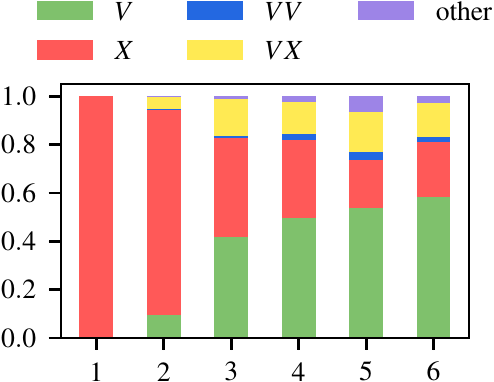}
  \vspace{-0.2cm}
  \caption{Left: the proportion of annotated images as a function of annotation time for \approachClassification{} and {\em standard} strategies. Middle: average episode duration for various batches of data. Right: the proportion of various annotation sequences in batches of data at \num{6} iteration.}
  \vspace{-0.2cm}
\label{fig:retraining}
\end{figure*}

In realistic settings, the detector becomes stronger with a growing amount of annotations.
Thus, to annotate bounding boxes with minimal cost, the object detector should be iteratively re-trained on previously annotated data.

\mypar{Horizontal re-training}

One way to introduce detector re-training is suggested by the box verification series technique~\cite{papadopoulos16cvpr}.
It starts with a given object detector $f_0$, typically trained on image-level labels using MIL.
In the first iteration, $f_0$ is applied to all images, and the highest scored detection $b_1$ in each image is sent for human verification.
After this, the detector is re-trained on all accepted boxes, giving a new detector $f_1$.
In the second iteration, $f_1$ is applied to all images where a proposed box was rejected, attempting to localize the objects again as $f_1$ is stronger than $f_0$ (re-localization phase). Afterwards, these new detections are sent for verification, and finally the detector is re-trained again.
The re-training, re-localization, and verification phases are iteratively alternated for a predefined number of iterations.
We refer to this method as {\em $V$-hor} in our experiments. It essentially corresponds to the original method~\cite{papadopoulos16cvpr}.

\mypar{Vertical re-training}

A different way to incorporate detector re-training is inspired by batch-mode active learning \cite{settles10survey}.
In this case, a subset of images $I_1$ (batch) is annotated until completion, by running box verification series in each image while keeping the initial detector $f_0$ fixed.
After this, the detector is re-trained on all boxes produced so far, giving $f_1$, and is then applied to the next batch $I_2$ to generate box proposals. The process iteratively moves from batch to batch until all images are processed.

\mypar{IAD with vertical re-training}

It is straightforward to apply vertical retraining to any fixed dialog strategy.
However, re-training the detector on more data increases the advantage of $V$ over $D$, so a truly adaptive strategy should change as the detector gets stronger.
We achieve this with the following procedure. 
At any given iteration $\tau$, we train dialog strategy \approachClassification{}$(I_\tau,f_{\tau-1})$ using boxes collected on $I_\tau$ and detector $f_{\tau-1}$.
\approachClassification{}$(I_\tau,f_{\tau-1})$ is applied with detector $f_\tau$ to collect new boxes on the next batch $I_{\tau+1}$. 
Note that \approachClassification{}$(I_\tau,f_{\tau-1})$ is trained with the help of detector $f_{\tau-1}$, but it is applied with the box proposals of detector $f_\tau$.
This procedure introduces a small discrepancy, but it is not important when detectors $f_\tau$ and $f_{\tau-1}$ are sufficiently similar, which is the case in the experiments below. 
To initialize the procedure we set $f_0$ to be a weakly supervised MIL detector and we annotate $I_1$ by manual box drawing $D$.

We set the desired quality of bounding boxes to high (i.e. $\alpha=0.7$) and we use Extreme Clicking for manual drawing.
We perform 6 re-training iterations with an increasingly large batch size:
$\left\vert{I_1}\right\vert = 3.125\%$, $\left\vert{I_2}\right\vert = 3.125\%$, $\left\vert{I_3}\right\vert = 6.25\%$, $\left\vert{I_4}\right\vert = 12.5\%$, $\left\vert{I_5}\right\vert = 25\%$, $\left\vert{I_6}\right\vert = 50\%$. 
This batching schedule is motivated by the fact that the gain in detector's performance after re-training is more noticeable when the previous training set is considerably smaller.

\mypar{Results}

Fig.~\ref{fig:retraining} (left) shows what proportion of boxes is collected as a function of total annotation time.
We compare \approachClassification{} against the strategy {\em $V$-hor} \cite{papadopoulos16cvpr}, and the standard fast drawing strategy $X$.
\approachClassification{} is able to annotate the whole dataset faster than any of the considered strategies.
Fig.~\ref{fig:retraining} (middle) shows the average episode duration in each batch.
By design, the annotation time for strategy $X$ is constant.
For {\em $V$-hor}, after the first re-training iteration (from a weakly supervised to supervised detector) the average annotation cost grows because only difficult images are left to be annotated.
On the contrary, annotation time for \IAD{} decreases with every new batch because dialogs become stronger and box verifications become more successful.

\mypar{Quality of boxes and resulting detector}
The data for training a detector and strategy in IAD includes both manually drawn boxes and boxes verified at IoU $ \ge 0.7$. 
More precisely, IAD data collection results in \num{44}\% drawn boxes and
\num{56}\% verified boxes. The quality of the verified boxes reaches \num{83}\% mIoU.
The detector trained on the boxes produced by IAD reaches \num{98}\% of the mAP of the detector trained on ground-truth boxes.

\mypar{Evolution of adaptive strategies}
\label{sec:analyses}

To gain better understanding of adaptive behaviour of \IAD{}, we study the composition of sequences of actions produced during labelling of each batch.
Fig.~\ref{fig:retraining} (right) shows the proportion of images that are labelled by $X$, $V$, $VX$, $VV$ and others sequences of actions.
At the beginning of the process (batch \num{2}), the vast majority of boxes is produced simply by asking for Extreme Clicking ($X$).
It means that \IAD{} learns that this is the best thing to do when the detector is weak. 
As the process continues, the detector gets stronger and \IAD{} selects more frequently series composed purely of box verifications ($V$,$VV$), and mixed series with both actions ($VX$).
Examples of the annotation dialogs are presented in the supplementary materials.
This experiment demonstrates that \IAD{} is capable of producing strategies that dynamically adapt to the change in problem property caused by the gradually improving detector. One cannot achieve this with any fixed strategy.

\section{Conclusion}

In this paper we introduced Intelligent Annotation Dialogs for the task of bounding box annotation. 
IAD automatically chooses a sequence of actions $V^kD$ that results in time-efficient annotations.
We presented two methods to achieve this.
The first method models the annotation time by predicting the acceptance probability for every box proposal.
The second method skips the modelling step and learns an efficient strategy directly from trial-and-error interactions.
In the extensive experimental evaluation IAD demonstrated competitive performance against various baselines and the ability to adapt to multiple problem properties. 
In future work we would like to model variable annotation time and context switches.

\pagebreak

{\small
\bibliographystyle{ieee}
\bibliography{shortstrings,calvin,new_bib_entries}

\begin{thebibliography}{10}\itemsep=-1pt

\bibitem{asada96ml}
M.~Asada, S.~Noda, S.~Tawaratsumida, and K.~Hosoda.
\newblock Purposive behavior acquisition for a real robot by vision-based
  reinforcement learning.
\newblock {\em Machine Learning}, 23(2):279--303, 1996.

\bibitem{bearman16eccv}
A.~Bearman, O.~Russakovsky, V.~Ferrari, and L.~Fei-Fei.
\newblock What's the point: Semantic segmentation with point supervision.
\newblock In {\em ECCV}, 2016.

\bibitem{bell15cvpr}
S.~Bell, P.~Upchurch, N.~Snavely, and K.~Bala.
\newblock Material recognition in the wild with the materials in context
  database.
\newblock In {\em CVPR}, 2015.

\bibitem{bellver16nipsws}
M.~Bellver, X.~G. i~Nieto, F.~Marques, and J.~Torres.
\newblock Hierarchical object detection with deep reinforcement learning.
\newblock In {\em NIPS Workshop on Deep Reinforcement Learning}, 2016.

\bibitem{bilen14bmvc}
H.~Bilen, M.~Pedersoli, and T.~Tuytelaars.
\newblock Weakly supervised object detection with posterior regularization.
\newblock In {\em BMVC}, 2014.

\bibitem{bilen16cvpr}
H.~Bilen and A.~Vedaldi.
\newblock Weakly supervised deep detection networks.
\newblock In {\em CVPR}, 2016.

\bibitem{biswas13cvpr}
A.~Biswas and D.~Parikh.
\newblock Simultaneous active learning of classifiers \& attributes via
  relative feedback.
\newblock In {\em CVPR}, 2013.

\bibitem{boykov01iccv}
Y.~Boykov and M.~P. Jolly.
\newblock Interactive graph cuts for optimal boundary and region segmentation
  of objects in {N}-{D} images.
\newblock In {\em ICCV}, 2001.

\bibitem{branson14cvpr}
S.~Branson, K.~Hj\"orleifsson, and P.~Perona.
\newblock Active annotation translation.
\newblock In {\em CVPR}, 2014.

\bibitem{branson10eccv}
S.~Branson, C.~Wah, F.~Schroff, B.~Babenko, P.~Welinder, P.~Perona, and
  S.~Belongie.
\newblock Visual recognition with humans in the loop.
\newblock In {\em ECCV}, 2010.

\bibitem{caicedo15iccv}
J.~C. Caicedo and S.~Lazebnik.
\newblock Active object localization with deep reinforcement learning.
\newblock In {\em ICCV}, pages 2488--2496, 2015.

\bibitem{castrejon17cvpr}
L.~Castrejon, K.~Kundu, R.~Urtasun, and S.~Fidler.
\newblock Annotating object instances with a polygon-rnn.
\newblock 2017.

\bibitem{cinbis14cvpr}
R.~Cinbis, J.~Verbeek, and C.~Schmid.
\newblock Multi-fold mil training for weakly supervised object localization.
\newblock In {\em CVPR}, 2014.

\bibitem{deselaers10eccv}
T.~Deselaers, B.~Alexe, and V.~Ferrari.
\newblock Localizing objects while learning their appearance.
\newblock In {\em ECCV}, 2010.

\bibitem{everingham10ijcv}
M.~Everingham, L.~Van~Gool, C.~K.~I. Williams, J.~Winn, and A.~Zisserman.
\newblock {The {PASCAL} Visual Object Classes ({VOC}) Challenge}.
\newblock {\em IJCV}, 2010.

\bibitem{haussmann17cvpr}
M.~Hau{\ss}mann, F.~Hamprecht, and M.~Kandemir.
\newblock Variational bayesian multiple instance learning with gaussian
  processes.
\newblock In {\em CVPR}, 2017.

\bibitem{jain16hcomp}
S.~Jain and K.~Grauman.
\newblock Click carving: Segmenting objects in video with point clicks.
\newblock In {\em Proceedings of the Fourth AAAI Conference on Human
  Computation and Crowdsourcing}, 2016.

\bibitem{jain13iccv}
S.~D. Jain and K.~Grauman.
\newblock Predicting sufficient annotation strength for interactive foreground
  segmentation.
\newblock In {\em ICCV}, 2013.

\bibitem{jayaraman17arxiv}
D.~Jayaraman and K.~Grauman.
\newblock Learning to look around.
\newblock {\em arXiv preprint arXiv:1709.00507}, 2017.

\bibitem{joshi09cvpr}
A.~J. Joshi, F.~Porikli, and N.~Papanikolopoulos.
\newblock Multi-class active learning for image classification.
\newblock In {\em CVPR}, 2009.

\bibitem{kantorov16eccv}
V.~Kantorov, M.~Oquab, M.~Cho, and I.~Laptev.
\newblock Contextlocnet: Context-aware deep network models for weakly
  supervised localization.
\newblock In {\em ECCV}, 2016.

\bibitem{kovashka11iccv}
A.~Kovashka, S.~Vijayanarasimhan, and K.~Grauman.
\newblock Actively selecting annotations among objects and attributes.
\newblock In {\em ICCV}, 2011.

\bibitem{openimages}
I.~Krasin, T.~Duerig, N.~Alldrin, V.~Ferrari, S.~Abu-El-Haija, A.~Kuznetsova,
  H.~Rom, J.~Uijlings, S.~Popov, A.~Veit, S.~Belongie, V.~Gomes, A.~Gupta,
  C.~Sun, G.~Chechik, D.~Cai, Z.~Feng, D.~Narayanan, and K.~Murphy.
\newblock Openimages: A public dataset for large-scale multi-label and
  multi-class image classification.
\newblock {\em Dataset available from
  \url{https://github.com/openimages/dataset}}, 2017.

\bibitem{levine16jmlr}
S.~Levine, C.~Finn, T.~Darrell, and P.~Abbeel.
\newblock End-to-end training of deep visuomotor policies.
\newblock {\em JMLR}, 17:1--40, 2016.

\bibitem{lin16cvpr}
D.~Lin, J.~Dai, J.~Jia, K.~He, and J.~Sun.
\newblock {ScribbleSup}: Scribble-supervised convolutional networks for
  semantic segmentation.
\newblock In {\em CVPR}, 2016.

\bibitem{mathe12eccv}
S.~Mathe and C.~Sminchisescu.
\newblock Dynamic eye movement datasets and learnt saliency models for visual
  action recognition.
\newblock In {\em ECCV}, 2012.

\bibitem{mettes16eccv}
P.~Mettes, J.~C. van Gemert, and C.~G. Snoek.
\newblock Spot on: Action localization from pointly-supervised proposals.
\newblock In {\em ECCV}, 2016.

\bibitem{michels05icml}
J.~Michels, A.~Saxena, and A.~Ng.
\newblock High speed obstacle avoidance using monocular vision and
  reinforcement learning.
\newblock In {\em ICML}, 2005.

\bibitem{mnih13nips}
V.~Mnih, K.~Kavukcuoglu, D.~Silver, A.~Graves, I.~Antonoglou, D.~Wierstra, and
  M.~Riedmiller.
\newblock Playing atari with deep reinforcement learning.
\newblock In {\em NIPS Deep Learning Workshop}. 2013.

\bibitem{nagaraja15iccv}
N.~S. Nagaraja, F.~R. Schmidt, and T.~Brox.
\newblock Video segmentation with just a few strokes.
\newblock In {\em ICCV}, 2015.

\bibitem{papadopoulos14eccv}
D.~P. Papadopoulos, A.~D.~F. Clarke, F.~Keller, and V.~Ferrari.
\newblock Training object class detectors from eye tracking data.
\newblock In {\em ECCV}, 2014.

\bibitem{papadopoulos17iccv}
D.~P. Papadopoulos, J.~R. Uijlings, F.~Keller, and V.~Ferrari.
\newblock Extreme clicking for efficient object annotation.
\newblock In {\em ICCV}, 2017.

\bibitem{papadopoulos17cvpr}
D.~P. Papadopoulos, J.~R. Uijlings, F.~Keller, and V.~Ferrari.
\newblock Training object class detectors with click supervision.
\newblock In {\em CVPR}, 2017.

\bibitem{papadopoulos16cvpr}
D.~P. Papadopoulos, J.~R.~R. Uijlings, F.~Keller, and V.~Ferrari.
\newblock We don't need no bounding-boxes: Training object class detectors
  using only human verification.
\newblock In {\em CVPR}, 2016.

\bibitem{parkash12eccv}
A.~Parkash and D.~Parikh.
\newblock Attributes for classifier feedback.
\newblock In {\em ECCV}, 2012.

\bibitem{ren15nips}
S.~Ren, K.~He, R.~Girshick, and J.~Sun.
\newblock Faster {R-CNN}: Towards real-time object detection with region
  proposal networks.
\newblock In {\em NIPS}, 2015.

\bibitem{Rother04-tdfixed}
C.~Rother, V.~Kolmogorov, and A.~Blake.
\newblock Grabcut: interactive foreground extraction using iterated graph cuts.
\newblock {\em SIGGRAPH}, 23(3):309--314, 2004.

\bibitem{russakovsky15ijcv}
O.~Russakovsky, J.~Deng, H.~Su, J.~Krause, S.~Satheesh, S.~Ma, Z.~Huang,
  A.~Karpathy, A.~Khosla, M.~Bernstein, A.~Berg, and L.~Fei-Fei.
\newblock {ImageNet} large scale visual recognition challenge.
\newblock {\em IJCV}, 2015.

\bibitem{russakovsky15cvpr}
O.~Russakovsky, L.-J. Li, and L.~Fei-Fei.
\newblock Best of both worlds: human-machine collaboration for object
  annotation.
\newblock In {\em CVPR}, 2015.

\bibitem{settles10survey}
B.~Settles.
\newblock Active learning literature survey.
\newblock {\em University of Wisconsin, Madison}, 2010.

\bibitem{siddiquie10cvpr}
B.~Siddiquie and A.~Gupta.
\newblock Beyond active noun tagging: Modeling contextual interactions for
  multi-class active learning.
\newblock In {\em CVPR}, 2010.

\bibitem{siva11iccv}
P.~Siva and T.~Xiang.
\newblock Weakly supervised object detector learning with model drift
  detection.
\newblock In {\em ICCV}, 2011.

\bibitem{su12aaai}
H.~Su, J.~Deng, and L.~Fei-Fei.
\newblock Crowdsourcing annotations for visual object detection.
\newblock In {\em AAAI Human Computation Workshop}, 2012.

\bibitem{szegedy16cvpr}
C.~Szegedy, V.~Vanhoucke, S.~Ioffe, J.~Shlens, and Z.~Wojna.
\newblock Rethinking the inception architecture for computer vision.
\newblock In {\em CVPR}, 2016.

\bibitem{vijayanarasimhan08nips}
S.~Vijayanarasimhan and K.~Grauman.
\newblock Multi-level active prediction of useful image annotations for
  recognition.
\newblock In {\em NIPS}, 2008.

\bibitem{vijayanarasimhan09cvpr}
S.~Vijayanarasimhan and K.~Grauman.
\newblock What's it going to cost you?: Predicting effort vs. informativeness
  for multi-label image annotations.
\newblock In {\em CVPR}, 2009.

\bibitem{vijayanarasimhan14ijcv}
S.~Vijayanarasimhan and K.~Grauman.
\newblock Large-scale live active learning: Training object detectors with
  crawled data and crowds.
\newblock {\em IJCV}, 108(1-2):97--114, 2014.

\bibitem{vondrick13ijcv}
C.~Vondrick, D.~Patterson, and D.~Ramanan.
\newblock Efficiently scaling up crowdsourced video annotation.
\newblock {\em IJCV}, 2013.

\bibitem{wah14cvpr}
C.~Wah, G.~Van~Horn, S.~Branson, S.~Maji, P.~Perona, and S.~Belongie.
\newblock Similarity comparisons for interactive fine-grained categorization.
\newblock In {\em CVPR}, 2014.

\bibitem{wang14cviu}
T.~Wang, B.~Han, and J.~Collomosse.
\newblock Touchcut: Fast image and video segmentation using single-touch
  interaction.
\newblock {\em CVIU}, 2014.

\bibitem{xu15cvpr}
J.~Xu, A.~G. Schwing, and R.~Urtasun.
\newblock Learning to segment under various forms of weak supervision.
\newblock In {\em CVPR}, 2015.

\bibitem{yao2012cvpr}
A.~Yao, J.~Gall, C.~Leistner, and L.~Van~Gool.
\newblock Interactive object detection.
\newblock In {\em CVPR}, 2012.

\bibitem{zhu17iccv}
Y.~Zhu, Y.~Zhou, Q.~Ye, Q.~Qiu, and X.~Jiao.
\newblock Soft proposal networks for weakly supervised object localization.
\newblock In {\em ICCV}, 2017.

\end{thebibliography}
}

\end{document}